\documentclass[conference]{IEEEtran}
\IEEEoverridecommandlockouts
\usepackage{amsmath,amssymb,amsfonts}
\usepackage{algorithmic}
\usepackage{graphicx}
\usepackage{textcomp}
\usepackage{xcolor}
\def\BibTeX{{\rm B\kern-.05em{\sc i\kern-.025em b}\kern-.08em
    T\kern-.1667em\lower.7ex\hbox{E}\kern-.125emX}}

\usepackage[square, numbers, sort&compress]{natbib}
\usepackage{booktabs}
\usepackage[hidelinks]{hyperref}
\usepackage{algorithm}
\usepackage{algorithmic}
\usepackage{caption}
\usepackage{subcaption}
\usepackage{float}
\usepackage{wrapfig}
\usepackage{pgf}
\usepackage{import}
\usepackage{tikz}
\usetikzlibrary{calc}
\usetikzlibrary{positioning}
\usetikzlibrary{shapes}
\usetikzlibrary{fit}

\usepackage{amsmath}
\usepackage{amssymb}
\usepackage{mathtools}
\usepackage{amsthm}

\usepackage{amsmath,amsfonts,bm}

\let\originalleft\left
\let\originalright\right
\renewcommand{\left}{\mathopen{}\mathclose\bgroup\originalleft}
\renewcommand{\right}{\aftergroup\egroup\originalright}

\def\eqref#1{equation~\ref{#1}}

\def\ceil#1{\lceil #1 \rceil}

\def\1{\bm{1}}

\def\rf{{\textnormal{f}}}
\def\rg{{\textnormal{g}}}

\def\vone{{\bm{1}}}

\def\vdelta{{\bm{\delta}}}
\def\va{{\bm{a}}}
\def\vb{{\bm{b}}}

\def\vg{{\bm{g}}}
\def\vh{{\bm{h}}}

\def\vs{{\bm{s}}}

\def\vx{{\bm{x}}}
\def\vy{{\bm{y}}}

\def\evs{{s}}

\DeclareMathAlphabet{\mathsfit}{\encodingdefault}{\sfdefault}{m}{sl}
\SetMathAlphabet{\mathsfit}{bold}{\encodingdefault}{\sfdefault}{bx}{n}

\def\sA{{\mathbb{A}}}

\def\sG{{\mathbb{G}}}

\def\sP{{\mathbb{P}}}

\def\sS{{\mathbb{S}}}

\newcommand{\E}{\mathbb{E}}

\newcommand{\abs}[1]{\left|#1\right|}

\newcommand{\normlzero}{L^0}

\newcommand{\normltwo}{L^2}

\newcommand{\loss}[1]{\mathcal{L}\left(#1\right)}

\newcommand{\trans}[1]{#1^{\mathsf{T}}}

\theoremstyle{plain}
\newtheorem{theorem}{Theorem}[section]

\newtheorem{lemma}[theorem]{Lemma}

\theoremstyle{definition}

\theoremstyle{remark}

\begin{document}

\title{Always-Sparse Training by Growing Connections with Guided Stochastic Exploration}

\author{\IEEEauthorblockN{Mike Heddes$^{1,*}$, Narayan Srinivasa$^2$, Tony Givargis$^1$, Alex Nicolau$^1$}
\IEEEauthorblockA{$^{1}$Department of Computer Science, University of California, Irvine, USA\\
$^{2}$Intel Labs, Santa Clara, USA\\
$^{*}$Corresponding author: mheddes@uci.edu}
}

\maketitle

\begin{abstract}%
The excessive computational requirements of modern artificial neural networks (ANNs) are posing limitations on the machines that can run them. 
Sparsification of ANNs is often motivated by time, memory and energy savings only during model inference, yielding no benefits during training. A growing body of work is now focusing on providing the benefits of model sparsification also during training. While these methods greatly improve the training efficiency, the training algorithms yielding the most accurate models still materialize the dense weights, or compute dense gradients during training. 
We propose an efficient, always-sparse training algorithm with excellent scaling to larger and sparser models, supported by its linear time complexity with respect to the model width during training and inference.
Moreover, our guided stochastic exploration algorithm improves over the accuracy of previous sparse training methods.
We evaluate our method on the CIFAR-10/100 and ImageNet classification tasks using ResNet, VGG, and ViT models, and compare it against a range of sparsification methods\footnote[3]{The source code is available at:\\ \url{https://github.com/mikeheddes/guided-stochastic-exploration}}. 

\end{abstract}


\section{Introduction}
\label{sec:introduction}

Artificial neural networks (ANNs) are currently the most prominent machine learning method due to their unparalleled performance in a wide range of applications, including computer vision~\cite{voulodimos2018deep, o2019deep}, natural language processing~\cite{young2018recent, otter2020survey}, and reinforcement learning~\cite{arulkumaran2017deep, schrittwieser2020mastering}, among many others~\cite{liu2017survey, wang2020recent, zhou2020graph}. 
To improve their representational powers, ANNs continue to increase in size~\cite{neyshaburrole, kaplan2020scaling}. 
Recent large language models, for example, have widths exceeding $10{,}000$ units and total parameter counts over 100 billion~\cite{brown2020language, rae2021scaling, chowdhery2022palm}.
However, with their increasing sizes, the memory and computational requirements to train and make inferences with these models become a limiting factor \cite{hwang2018computational, ahmed2020democratization}. 

\begin{figure}[t!]
    \centering
    \scalebox{0.85}{
    \begin{tikzpicture}

\tikzstyle{node}=[circle, draw, line width=1pt, minimum size=0.75cm, inner sep=0pt, font=\small]
\tikzstyle{active}=[line width=1.5pt]
\tikzstyle{inactive}=[dashed, gray]

\node[node, fill=blue!30] (L1) {1};
\node[node, fill=blue!30, below=0.4cm of L1] (L2) {2};
\node[node, fill=blue!30, below=0.4cm of L2] (L3) {3};

\node[node, fill=red!30, right=2cm of L1] (R1) {1};
\node[node, fill=red!30, below=0.4cm of R1] (R2) {2};
\node[node, fill=red!30, below=0.4cm of R2] (R3) {3};

\draw[active] (L1) -- (R2);
\draw[active] (L2) -- (R2);
\draw[active] (L2) -- (R3);
\draw[active] (L3) -- (R1);

\draw[inactive] (L1) -- (R1);
\draw[inactive] (L1) -- (R3);
\draw[inactive] (L2) -- (R1);
\draw[inactive] (L3) -- (R2);
\draw[inactive] (L3) -- (R3);

\node[right=0.3cm of R2] {\Large$\rightarrow$};

\node[node, fill=blue!30, right=4cm of L1] (RL1) {1};
\node[node, fill=blue!30, below=0.4cm of RL1] (RL2) {2};
\node[node, fill=blue!30, below=0.4cm of RL2] (RL3) {3};

\node[node, fill=red!30, right=2cm of RL1] (RR1) {1};
\node[node, fill=red!30, below=0.4cm of RR1] (RR2) {2};
\node[node, fill=red!30, below=0.4cm of RR2] (RR3) {3};

\draw[active] (RL1) -- (RR2);
\draw[active] (RL2) -- (RR3);
\draw[active] (RL3) -- (RR1);
\draw[active] (RL3) -- (RR2);

\draw[inactive] (RL1) -- (RR1);
\draw[inactive] (RL2) -- (RR2);
\draw[inactive] (RL1) -- (RR3);
\draw[inactive] (RL2) -- (RR1);
\draw[inactive] (RL3) -- (RR3);

\node[below=0.5cm of L3, xshift=1.5cm] (Aleft) {$A = [(1,2), (2,2), (2,3), (3,1)]$};
\node[below=0.0cm of Aleft] (thetaLeft) {$\theta = [-0.48, -0.14, -0.46, 0.73]$};

\node[below=0.5cm of RL3, xshift=1.5cm] (Aright) {$A = [(1,2), (2,3), (3,1), (3,2)]$};
\node[below=0.0cm of Aright] (thetaRight) {$\theta = [-0.48, -0.46, 0.73, 0.00]$};

\node[above=0.5cm of L1, xshift=-0.5cm] (legend1) {};
\draw[active] (legend1) -- ++(0.75,0) node[right, black] {Active connection};
\node[above=0.5cm of RL1, xshift=-0.5cm] (legend2) {};
\draw[inactive] (legend2) -- ++(0.75,0) node[right, black] {Inactive connection};

\end{tikzpicture}}
    \caption{Illustration of the connections before (left) and after (right) a prune and grow step in dynamic sparse training on a 3-by-3 feedforward layer. The sparse model parameters are represented by the active connections $A$ and their weights $\theta$. Connection $(2,2)$ is pruned because its weight has the lowest magnitude. The connection $(3, 2)$ is grown and its weight will hereafter be optimized with stochastic gradient descent.}
    \label{fig:prune-grow-illustration}
\end{figure}

A large body of work has addressed problems that arise from the immense model sizes~\cite{reed1993pruning, gale2019state, blalock2020state, hoefler2021sparsity}. Many studies look into sparsifying the model weights based on the observation that the weight distribution of a dense model is often concentrated around zero at the end of training, indicating that the majority of weights contribute little to the function being computed \cite{han2015learning}. Using sparse matrix representations, the model size and the number of floating-point operations (FLOPs) can be dramatically reduced.
Moreover, previous work has found that for a fixed memory or parameter budget, larger sparse models outperform smaller dense models \cite{zhu2017prune, kalchbrenner2018efficient}.

Early works in ANN sparsity removed connections, a process called \textit{pruning}, of a trained dense model based on the magnitude of the weights \cite{janowsky1989pruning, strom1997sparse}, resulting in a more efficient model for inference. Later works improved upon this technique \cite{guo2016dynamic, dong2017learning, yu2018nisp}, but all require at least the cost of training a dense model, thus providing no efficiency benefits during training. Consequently, the resulting sparse models are limited in size by the largest dense model that can be trained on a given machine.

In light of the aforementioned limitation, the lottery ticket hypothesis (LTH), surprisingly, hypothesized that there exists a subnetwork within a dense over-parameterized model that, when trained with the same initial weights, will result in a sparse model with accuracy comparable to that of the dense model \cite{frankle2018lottery}. However, the proposed method for finding a \textit{winning ticket} within a dense model is very computationally intensive, as it requires training the dense model (typically multiple times) to obtain the subnetwork. Moreover, later work weakened the hypothesis for larger ANNs \cite{frankle2019stabilizing}. Despite this, it was still an important catalyst for new methods that aim to find the winning ticket more efficiently.

Efficient methods for finding a winning ticket can be categorized as pruning before training and dynamic sparse training (DST). The before-training methods prune connections from a randomly initialized \textit{dense} model \cite{lee2018snip,tanaka2020pruning}. This means that they still require a machine capable of representing and computing with the dense model.
In contrast, DST methods start with a randomly initialized \textit{sparse} model and dynamically change the connections throughout training while maintaining the overall sparsity~\cite{mocanu2018scalable, mostafa2019parameter, evci2020rigging}, as illustrated in Figure~\ref{fig:prune-grow-illustration}.
In practice, DST methods achieve higher accuracy than pruning before training methods (see Section~\ref{sec:comparison-related-work}). 
The first DST method, sparse evolutionary training (SET), periodically prunes the connections with the lowest weight magnitude and \textit{grows} new connections uniformly at random \cite{mocanu2018scalable}. 
Subsequent work, such as RigL and Top-KAST \cite{evci2020rigging, jayakumar2020top}, has improved the accuracy of SET.
However, \textit{RigL periodically computes dense gradients}, and while Top-KAST always computes sparse gradients, \textit{it still materializes dense weights during training}.

Instead, we present a DST algorithm that is \textit{always sparse}, i.e., both the gradients and the weights are sparse at all times. 
The training and inference cost of our method therefore scales with the number of non-zero weights, like SET, but obtains substantially higher accuracy, on par with, and at times even exceeding, that of Top-KAST and RigL.
We improve the accuracy of previous methods using a simple, yet effective method that we call \textit{guided stochastic exploration} (GSE). In short, when new connections are added, our method first efficiently samples a subset of inactive connections for exploration. 
It is then guided by growing the connections with the largest gradient magnitude from the sampled subset. 
The key insight is that when the size of the subset is on the same order as the number of active connections, that is, the weight sparsity and gradient sparsity are similar, then the accuracy of our method surpasses that of previous sparse training methods. 
Our method has a time complexity of $O(n + N)$ with respect to the width of the model $n$, where $N \leq n^2$ is the number of non-zero weights.
When using the Erdős–Rényi sparse model initialization \cite{mocanu2018scalable}, the time complexity is reduced to only $O(n)$. 
We evaluated the classification accuracy of our method on the CIFAR-10/100 and ImageNet datasets using ResNet, VGG, and ViT models and compared it with a range of sparsification methods at 90\%, 95\%, and 98\% sparsity, indicating the percentage of zero weights. 
Our method promises to enable the training of larger and sparser models.

\section{Related work}
\label{sec:related_work}

A wide variety of methods have been proposed that aim to reduce the size of ANNs, such as dimensionality reduction of the model weights \citep{jaderberg2014speeding, novikov2015tensorizing}, and weight quantization~\citep{gupta2015deep, mishra2018wrpn}. 
However, this section specifically covers model sparsification methods as they are the most related to our work.
Following Wang et al. \citep{wang2019picking}, the sparsification methods are categorized as: pruning after training, pruning during training, pruning before training, and dynamic sparse training.

\paragraph{After training} The first pruning algorithms operated on dense trained models, they prune the connections with the smallest weight magnitude \citep{janowsky1989pruning, thimm1995evaluating, strom1997sparse, han2015learning}. This method was later generalized to first-order \citep{mozer1988skeletonization, karnin1990simple, molchanov2019pruning, molchanov2019importance} and second-order \citep{lecun1989optimal, hassibi1992second, frantar2021m} Taylor polynomials of the loss with respect to the weights. These methods can be interpreted as calculating an importance score for each connection based on how its removal will affect the loss~\citep{guo2016dynamic, dong2017learning, yu2018nisp}. 

\paragraph{During training} Gradual pruning increases the sparsity of a model during training untill the desired sparsity is reached \citep{zhu2017prune, liu2021sparse}. Kingma et al. \citep{kingma2015variational} introduced variational dropout which adapts the dropout rate of each unit during training, Molchanov et al. \citep{molchanov2017variational} showed that pruning the units with the highest dropout rate is an effective way to sparsify a model. Louizos et al. \citep{louizos2018learning} propose a method based on the reparameterization trick that allows to directly optimize the $\normlzero$ norm, which penalizes the number of non-zero weights. Alternatively, DeepHoyer is a differentiable regularizer with the same minima as the $\normlzero$ norm \citep{yang2019deephoyer}. Lastly, Powerpropagation is a sparsity-promoting regularizer that raises weights to a constant power \citep{schwarz2021powerpropagation}.

\paragraph{Before training} The lottery ticket hypothesis (LTH) \citep{frankle2018lottery, frankle2019stabilizing} started a line of work that aims to find a sparse model by pruning a dense model before training \citep{liu2018rethinking}. SNIP uses the sensitivity of each connection to the loss as the importance score of a connection \citep{lee2018snip}. GraSP optimizes gradient flow to accelerate training \citep{wang2019picking}, however, Lubana et al. \citep{lubana2020gradient} argue that it is better to preserve gradient flow instead. Gradient flow refers to the norm of the gradients after pruning. Tanaka et al. \citep{tanaka2020pruning} highlight a problem in the aforementioned methods: they suffer from layer collapse in high-sparsity regimes, that is, during the pruning phase all the connections of a layer can be removed, making the model untrainable. They propose SynFlow, which prevents layer collapse by calculating how each connection contributes to the flow of information using a path regularizer \citep{neyshabur2015norm}, similar to Lee et al. \citep{lee2019signal}.

\paragraph{Dynamic sparse training} The methods in this last category, including ours, start with a randomly initialized sparse model and dynamically change the connections throughout training while maintaining the same sparsity. This involves periodically pruning a fraction of active connections and growing the same number of inactive connections. This process is illustrated in Figure~\ref{fig:prune-grow-illustration}. SET was the first method and used magnitude pruning and random growing of connections \citep{mocanu2018scalable}. DeepR assigns a fixed sign to each connection at initialization and prunes those connections whose sign would change during training \citep{bellec2018deep}. DSR prunes connections with the lowest global weight magnitude and uses random growing, allowing the connections to be redistributed over the layers and used in layers where they contribute the most \citep{mostafa2019parameter}. SNFS improves upon the previous methods by using an informed growing criteria: it grows the connections with the largest gradient momentum magnitude \citep{dettmers2019sparse}. RigL makes this more efficient by periodically calculating the dense gradients \citep{evci2020rigging}. 
However, calculating the dense gradients is still a limitation for training very large and sparse models.
Top-KAST addresses this limitation by sparsely updating the dense model using only the connections with the largest weight magnitude \citep{jayakumar2020top}. Although Top-KAST uses sparse gradients, it requires maintaining and updating the dense model from which the sparse model is periodically reconstructed. Moreover, we will show in Section~\ref{sec:comparison-related-work} that Top-KAST achieves lower accuracy at higher sparsities.

\section{Guided stochastic exploration}
\label{sec:method}

Our main contribution is an efficient dynamic sparse training algorithm which is always sparse, that is, at no point is the dense model materialized, all forward passes use sparse weights, and it exclusively calculates sparse gradients.
The sparsity is maintained even while changing the connections because our method only computes gradients for a subset of the inactive connections. 
This subset is randomly sampled using an efficient sampling procedure during each round of growing and pruning.
The connections within the subset with the largest gradient magnitudes are then selected to grow.
We call our method \textit{guided stochastic exploration} (GSE), referring to the stochastic sampling process of the subset and the selection of the connections by largest gradient magnitude, which provides guidance during the exploration.
Although we are not the first to present an always-sparse training method, for example, SET also achieves this, among these methods, ours obtains the highest accuracy---equal or higher than less efficient methods like RigL and Top-KAST. 

\begin{figure}[h]
    \centering
    \begin{tikzpicture}[scale=1, transform shape, font=\small]

    \tikzset{set/.style={draw,circle,inner sep=0pt,align=center}}
    
    \node[set,dotted,text width=3cm,label={[below=75pt of all, xshift=35pt]$W$}] (prune) at (3.5,-0.95) (all) {};
    \node[set,fill=gray!30,text width=1.5cm, label={[below=25pt of active, xshift=0pt]$A$}] (active) at (3.5,-0.3)  {};
    \node[set,fill=white,text width=0.6cm] (prune) at (3.5,0) {$P$};
    \node[set, text width=1.2cm,label={[below=5pt of sub, xshift=23pt]$S$}] (sub) at (3.5,-1.75) {};
    \node[set,fill=gray!30,text width=0.6cm] (grow) at (3.5,-1.9) {$G$};

    \node[right=0.75cm of all, text width=0.3cm] (symb) {$W$\\$A$\\$S$\\$P$\\$G$};
    \node[right=0cm of symb, text width=5cm] (legend) {All connections\\Active connections\\Subset of inactive connections\\Pruned connections\\Grown connections};

\end{tikzpicture}
    \caption{
    Illustration of the relations between the connection sets. 
    The solid area becomes the next active set. The set $W$ is dotted to indicate that it is never materialized.
    }
    \label{fig:weight-sets}
\end{figure}

We will use the following notation and definitions. With $W = \{1, 2, \dots, n\} \times \{1, 2, \dots, n\}$ we denote the set of all possible connections, which is never materialized; $A \subseteq W$ is the set of active connections whose weights are optimized with gradient descent, while the inactive connections $W \setminus A$ have an implicit weight of zero; $S \subseteq (W \setminus A)$ is a subset of inactive connections; $G \subseteq S$ is the set of grown connections; and $P \subseteq A$ is the set of pruned connections.
Figure~\ref{fig:weight-sets} illustrates the relations between the various connection sets. 
Note that we abuse the notation of connection sets to be per layer or global interchangeably.

GSE is designed to incorporate the strengths of both SET and RigL.
While RigL typically outperforms SET, Liu et al. \citep{liu2021we} showed that SET can outperform RigL when training is significantly extended. This is because RigL only explores a fraction of the connections compared to SET.
The two baseline methods can be regarded as opposing extremes in their exploration strategy.
On one end, SET conducts unbiased exploration, enabling it to be highly efficient. On the opposite end, RigL employs a greedy strategy by selecting connections which appear to get the largest weight magnitudes.
The exploration dynamics of GSE are, in turn, determined by the interaction of the subset and the grow set. 
When the subset is much larger than the grow set, in the limit including all possible connections, then the exploration is nearly greedy since the sampling process has little influence over which connections are selected. 
On the other hand, when the sets are equal in size, then the grown connections are fully determined by the subset sampling process. 
The exploration strategy of GSE is thus a hybrid of random exploration guided by the selection of the largest gradient magnitude.

In the remainder of this section, we will discuss which probability distributions were considered for sampling the subset and how to efficiently sample from them in the case of feedforward layers. 
Our method naturally applies to convolutional layers by interpreting them as batched feedforward layers. See Appendix~\ref{sec:conv-layer-subset} for an extended discussion.
We conclude this section with our complete dynamic sparse training algorithm and its time complexity.

\subsection{Efficient subset sampling}
\label{sec:subset-sampling}

For our method to be always-sparse, it is critical that we can sample a subset of the inactive connections during the grow step without explicitly representing a probability for each possible connection.
Note that a connection $(a, b) \in W$ is fully specified by the input and output units it connects. Thus, to ensure an efficient sampling process, we can sample connections at random by independently sampling from distributions over the input and output units of a layer, which induces an implicit joint probability distribution over all the connections in that layer. 
With this process the probability of sampling a connection is never explicitly represented, just the probability of sampling each unit is represented which requires only linear, instead of quadratic, memory and $O(n)$ time to draw $n$ samples using the alias method~\citep{vose1991linear}. 

Formally, connections are sampled from the discrete probability distributions over the input and output units of the $l$-th layer: $a_i \sim \rf_l$ and $b_i \sim \rg_l$.
Since this process could sample duplicate or active connections, the set difference $S \gets S \setminus A$ with the active set is computed in $O(n)$ time using hash tables, resulting in a subset of connections whose size is upper bounded by the number of samples drawn, containing only inactive connections.

To sample the subset $S$, perhaps the simplest distribution is uniform over the set of all possible connections $W$. This reflects an unbiased exploration of the connections.
However, the distributions $\rf$ and $\rg$ do not have to be uniform, it might instead be useful to bias the subset towards connections that will contribute more at decreasing the loss, that is, those that have a large gradient magnitude. 
To this end, we investigate two other distributions. The first distribution is called GraBo, its probabilities are proportional to an upper bound of the gradient magnitude, given as follows: 
\begin{align*}
    \rf_l \coloneqq \frac{\abs{\vh_{l-1}}\vone}{\trans{\vone}\abs{\vh_{l-1}}\vone}, \quad \rg_l \coloneqq \frac{\abs{\vdelta_{l}}\vone}{ \trans{\vone}\abs{\vdelta_{l}}\vone}
\end{align*}
where $\vone$ is an all-ones vector, and $\vh_{l}, \vdelta_{l} \in \mathbb{R}^{n_{l}\times B}$ are the activations and gradients of the loss at the output units of the $l$-th layer, respectively. Here, the batch size is denoted as $B$ and $n_{l}$ is the number of units in the $l$-th layer.
The second distribution is called GraEst, its probabilities are proportional to an unbiased estimate of the gradient magnitude, given as follows:
\begin{align*}
        \rf_{l} \coloneqq \frac{\abs{\vh_{l-1}\vs}}{\trans{\vone}\abs{\vh_{l-1}\vs}}, \quad \rg_{l} \coloneqq \frac{\abs{\vdelta_{l}\vs}}{ \trans{\vone}\abs{\vdelta_{l}\vs}}
\end{align*}
where $\vs \in \{-1, +1\}^B$ is a vector of random signs. 
To compute the probabilities, both distributions can reuse the computation of the activations and the gradients during the forward and backward pass, respectively.
A more detailed discussion of the distributions is provided in Appendix~\ref{sec:sampling-distributions}.
In Section~\ref{sec:size-subset}, we assess which of these probability distributions is the most appropriate for sparse training.

\subsection{Dynamic sparse training}

Our complete dynamic sparse training algorithm starts by initializing each layer as a random bipartite graph using the Erdős–Rényi random graph generation algorithm~\citep{erdos59a}. The number of active connections $\abs{A}$ of the $l$-th layer is given by $\ceil{\epsilon(n_{l-1} + n_{l})}$, where $\epsilon > 0$ is a parameter that controls the sparsity of the model \citep{mocanu2018scalable}. This initialization ensures that the number of active connections in a layer is proportional to its width, that is, $\abs{A} = O(n)$. Moreover, the Erdős–Rényi initialization also achieves better accuracy than uniform sparsity assignment \citep{evci2020rigging}, as verified in Appendix~\ref{sec:layer-sparsity-distribution}. Once the initial connections are sampled, the weights are initialized so that the activations are normally distributed \citep{evci2019difficulty}.

\begin{algorithm}
\caption{Efficiently growing and pruning connections}
\label{alg:growing}
\begin{algorithmic}
\REQUIRE training step $t$, prune-grow schedule $\alpha, T_{\mathrm{end}}$, probability distributions $\rf$ and $\rg$, subset sampling factor $\gamma$, set of active connections $\sA$, and their weights $\theta$.
\STATE $\sS \gets \mathrm{sample\_connections}(\rf, \rg, \ceil{\gamma\abs{\sA}}) \setminus \sA$ %
\STATE $\alpha_t \gets \mathrm{cosine\_decay}(t; \alpha, T_{\mathrm{end}})$ %
\STATE $k \gets \min(\ceil{\alpha_t\abs{\sA}}, \abs{\sS})$ %
\STATE $\sG \gets \mathrm{topk}(\abs{\mathrm{grad}(\sS)}, k)$ %
\STATE $\sP \gets \mathrm{topk}(-\abs{\theta}, k)$ %
\STATE $\sA \gets (\sA \setminus \sP) \cup \sG$ %
\STATE $\theta \gets \mathrm{update\_weights}(\theta, \sA, \sP, \sG)$ %
\end{algorithmic}
\end{algorithm}

During training, the weights of the active connections are optimized using stochastic gradient descent (SGD). Every $T$ training steps new connections are grown, which are added with a weight of zero, and the same number of connections are pruned. We provide pseudocode for our pruning and growing procedure in Algorithm~\ref{alg:growing}, which operates on a single layer for simplicity. The maximum fraction of pruned active connections is cosine annealed from $\alpha \in (0, 1)$ at the start of training to zero at $T_{\mathrm{end}}$.
This ensures that there is more exploration in the early stages of training.

To determine the set of connections to grow $G$, we start by evaluating one of the probability distributions described in Section~\ref{sec:subset-sampling}. Then, $\ceil{\gamma\abs{A}}$ connections are sampled from the input and output distributions, where $\gamma > 0$ is a hyperparameter that adjusts the number of samples. The subset $S$ contains all randomly sampled connections that are not in the active set $A$. Now that the subset of inactive connections is determined, their gradient magnitude is computed. The grow set contains the largest $k$ gradient magnitude connections from the subset, where $k$ is the minimum between $\ceil{\alpha_t\abs{A}}$ and $\abs{S}$ to ensure that the grow set is not larger than the subset. The $k$ largest elements can be found in $O(n)$ time using the introspective selection algorithm~\citep{musser1997introspective}.

The pruned set $P$, in turn, contains the $k$ active connections with the smallest weight magnitude, also determined in $O(n)$ time. 
Although magnitude pruning has been observed to cause layer collapse among pruning before training methods,
the training process implicitly regularizes the weight magnitudes which minimizes the risk of layer collapse for DST methods like ours \citep{tanaka2020pruning}. 
Lastly, the active set and the weights are updated to reflect the changed connections. Pseudocode for the entire training process is provided in Appendix~\ref{sec:pseudocode}. It is important to stress that the dense model is not materialized at any point, only the active connections are represented, and all the operations are always sparse.

\section{Experiments}
\label{sec:experiments}

\begin{figure*}[t!]
    \centering
    \scalebox{0.9}{\subimport*{figures/}{subset-size-s98.pgf}}
    \caption{Accuracy of each distribution while increasing the number of subset samples (bounding the size of the subset) compared against RigL at 98\% sparsity.}
    \label{fig:subset-mul-98}
\end{figure*}

We evaluated our method and compared it to baselines on the CIFAR-10 and CIFAR-100 datasets \citep{krizhevsky2009learning}, in addition to the ILSVRC-2012 ImageNet dataset \citep{deng2009imagenet}. 
To establish how each method compares with different model architectures, we experimented with two convolutional neural networks (CNNs), ResNet-56~\citep{he2016deep} and a 4 times downscaled version of VGG-16 \citep{simonyan2014very}, in addition to ViT~\citep{dosovitskiy2020image} modified for small datasets~\citep{beyer2022better} so that all models have roughly the same number of parameters. 
The bias and normalization parameters are kept dense since they only contribute marginally to the size and cost of a model. 
Like Wang et al. \citep{wang2019picking}, we evaluated all methods at 90\%, 95\%, and 98\% sparsity.
The high levels of sparsity emphasize the differences between the methods.
In addition, in Appendix~\ref{sec:sparse-benchmark}, we show that sparse matrix formats for unstructured sparsity become more efficient for matrix multiplication from 90\% sparsity compared to dense matrix formats, further motivating the need for such high levels of sparsity.
We repeat each experiment three times with random initializations. The plots report the mean and the 95th percentile, while the tables report the mean and standard deviation. We adopt the implementation of the baselines based on their available code.

All experiments use the same optimization settings to isolate the differences in sparsification. Following Evci et al. \citep{evci2020rigging} and Lubana et al. \citep{lubana2020gradient}, we use SGD with a momentum coefficient of 0.9, an $\normltwo$ regularization coefficient of 0.0001, and an initial learning rate of 0.1 which is dropped by a factor of 10 at epochs 80 and 120. The ViT experiments have one exception: they use an initial learning rate of 0.01. We use a batch size of 128 and train for 200 epochs. We also apply standard data augmentation to the training data, described in Appendix~\ref{sec:data-augmentation}.

The LTH procedure proposed by Frankle et al. \citep{frankle2018lottery}, denoted as Lottery in the results, uses iterative pruning with iterations at 0\%, 50\%, 75\%, 90\%, 95\% and 98\% sparsity. The gradual pruning method by Zhu et al. \citep{zhu2017prune}, denoted as Gradual in the results, reaches the target sparsity at the second learning rate drop and prunes connections every 1000 steps. Following Tanaka et al. \citep{tanaka2020pruning}, the pruning before training methods use a number of training samples equal to ten times the number of classes to prune the dense model to the target sparsity. Except for SynFlow, which uses 100 iterations with an exponential decaying sparsity schedule to prune the dense model. All DST methods use the same update schedule: the connections are updated every 1000 steps, and similar to Evci et al. \citep{evci2020rigging}, the fraction of pruned active connections $\alpha$ is cosine annealed from 0.2 to 0.0 at the second learning rate drop. This is because Liu et al. \citep{liu2021sparse} showed that DST methods struggle to recover from pruning when the learning rate is low. The pruning and growing procedures are applied globally for all methods, maintaining the overall sparsity. This approach generally outperforms local, layer-wise sparsity \citep{mostafa2019parameter}.

\subsection{Effect of subset sample size}
\label{sec:size-subset}

First, we want to determine how the size of the subset affects the accuracy. In addition, we are interested in comparing the accuracy obtained by the three distributions: uniform, gradient upper bound (GraBo), and gradient estimate (GraEst). The number of subset samples (which is a bound on the subset size) is set proportional to the number of active connections $\abs{A}$, with $\abs{S} \leq \ceil{\gamma\abs{A}}$ and $\gamma \in \{0.25, 0.5, 1, 1.5, 2\}$. 
The results are shown in the first three columns of Figure~\ref{fig:subset-mul-98} and include the baseline accuracy of RigL for comparison. 

We see that the accuracy of RigL is consistently matched, or even surpassed, when the number of subset samples is equal to the number of active connections, that is, $\gamma = 1$. 
This indicates that GSE is more effective at exploring the network connections than both SET and RigL.
We observe the same overall trend across all distributions, showing that the success of our method is not reliant on the specific characteristics of the sampling distribution.
Moreover, the accuracy quickly peaks as $\gamma>1$, validating that it is sufficient for the subset size to be on the same order as the active set size. 

We expect that as the subset size of GSE increases, the accuracy eventually drops back to the level of RigL because the largest gradient magnitude selection then becomes the dominant criteria.
To this end, we experimented with increasing the number of samples to $\gamma=100$ in the last column of Figure~\ref{fig:subset-mul-98}, note the logarithmic x-axis. 
The observations align precisely with our expectations, which further underscores that greedy exploration leads to inferior solutions when compared to the hybrid exploration strategy of GSE.

\begin{table*}[t!]
    \setlength{\tabcolsep}{5pt}
    \centering
    \caption{Accuracy comparison with related work on the CIFAR datasets.}
    \label{tab:method-comparison}
    \begin{sc}
    \begin{tabular}{l|ccc|ccc|ccc|ccc|ccc|ccc}
        \toprule
        Dataset & \multicolumn{9}{c|}{CIFAR-10} & \multicolumn{9}{c}{CIFAR-100} \\
        \midrule
        Model & \multicolumn{3}{c|}{ResNet-56} & \multicolumn{3}{c|}{ViT} & \multicolumn{3}{c|}{VGG-16} & \multicolumn{3}{c|}{ResNet-56} & \multicolumn{3}{c|}{ViT} & \multicolumn{3}{c}{VGG-16}\\
        Dense & \multicolumn{3}{c|}{92.7} & \multicolumn{3}{c|}{85.0} & \multicolumn{3}{c|}{89.0} & \multicolumn{3}{c|}{70.1} & \multicolumn{3}{c|}{58.8} & \multicolumn{3}{c}{62.3}\\
        Sparsity & 90\% & 95\% & 98\% & 90\% & 95\% & 98\% & 90\% & 95\% & 98\% & 90\% & 95\% & 98\% & 90\% & 95\% & 98\% & 90\% & 95\% & 98\%\\
        \midrule
        Lottery          & 89.8 & 87.8 & 84.6 & 85.4 & 85.1 & 72.3 & 88.2 & 86.7 & 83.3 & 63.6 & 58.6 & 50.4 & 57.2 & 57.0 & 42.6 & 55.9 & 49.7 & 37.1 \\
        Gradual & 91.3 & 89.9 & 88.0 & 83.5 & 79.3 & 72.2 & 89.1 & 88.2 & 85.8 & 67.2 & 64.4 & 57.1 & 55.6 & 53.3 & 43.6 & 60.1 & 57.0 & 44.5 \\ 
        \midrule
        \midrule
        Static           & 90.1 & 88.3 & 84.4 & 79.1 & 75.3 & 68.7 & 87.1 & 84.5 & 79.0 & 63.5 & 55.4 & 36.6 & 53.9 & 48.0 & 38.0 & 54.1 & 44.2 & 27.8 \\
        SNIP             & 89.6 & 87.8 & 82.9 & 75.1 & 70.4 & 63.8 & 88.3 & 86.8 & 10.0 & 62.6 & 55.5 & 41.7 & 49.7 & 44.1 & 34.6 & 52.3 & 39.1 & 1.0 \\
        GraSP            & 89.7 & 88.7 & 84.6 & 62.2 & 65.3 & 62.3 & 86.4 & 85.3 & 82.0 & 61.8 & 56.9 & 42.3 & 40.5 & 39.6 & 32.6 & 53.4 & 45.7 & 30.1 \\
        SynFlow          & 90.3 & 88.0 & 83.8 & 78.6 & 75.7 & 69.7 & 87.5 & 85.4 & 78.7 & 60.5 & 50.6 & 31.4 & 52.9 & 48.5 & 37.9 & 52.1 & 41.8 & 24.8 \\
        \midrule
        RigL             & 90.6 & 89.5 & 86.7 & 79.9 & 76.2 & 70.6 & 88.5 & 87.1 & 83.7 & 65.7 & 62.3 & 52.8 & \textbf{54.6} & \textbf{49.0} & 40.4 & 57.0 & 51.8 & 39.9 \\
        Top-KAST & 89.8 & 88.1 & 85.2 & 75.0 & 68.3 & 55.0 & 86.9 & 84.8 & 80.5 & 62.5 & 58.3 & 40.4 & 51.1 & 43.3 & 28.2 & 54.4 & 47.0 & 32.2 \\
        \midrule
        SET              & 90.2 & 88.8 & 85.5 & 79.4 & 75.2 & 68.8 & 86.8 & 84.8 & 79.9 & 65.0 & 60.5 & 49.7 & 54.3 & 48.2 & 38.7 & 54.7 & 47.9 & 32.4 \\
        GSE (\textit{ours})    & \textbf{91.0} & \textbf{89.9} & \textbf{87.0} & \textbf{80.0} & \textbf{76.4} & \textbf{70.7} & \textbf{88.6} & \textbf{88.1} & \textbf{85.1} & \textbf{66.0} & \textbf{62.6} & \textbf{54.3} & 54.4 & 48.8 & \textbf{40.7} & \textbf{57.9} & \textbf{52.3} & \textbf{40.8} \\
        \bottomrule
    \end{tabular}
    \end{sc}
\end{table*}

Surprisingly, uniform sampling demonstrates strong performance while we hypothesized that biased distributions would be beneficial for training.
This shows that exploring connections which, at first, do not seem beneficial for training, can in fact be quite helpful.
It would be interesting for future research to investigate exactly which training dynamics are enabled by uniform sampling.
Since uniform sampling is also the most efficient distribution, we use it as our method in the following experiments.

\subsection{Comparison with related work}
\label{sec:comparison-related-work}

We compare our method with a wide range of sparsification methods in Table~\ref{tab:method-comparison} on CIFAR-10 and CIFAR-100. The best accuracy among the sparse training methods is in bold. The Static method denotes training a static random sparse model. 
The other baselines are Lottery \citep{frankle2018lottery}, Gradual \citep{zhu2017prune}, SNIP \citep{lee2018snip}, GraSP \citep{wang2019picking}, SynFlow \citep{tanaka2020pruning}, SET \citep{mocanu2018scalable}, RigL \citep{evci2020rigging}, and Top-KAST \citep{jayakumar2020top}. Our results use GSE with uniform sampling and $\gamma = 1$.

Among the sparse training methods (from Static downwards in Table~\ref{tab:method-comparison}) our method outperforms the other methods consistently over datasets, sparsities, and model architectures. While at 90\% sparsity all methods achieve comparable accuracy, at the extreme sparsity rate of 98\%, differences between methods become more evident. At 98\% sparsity our method outperforms all the before training sparsification methods by a significant margin, 4.7\% on average. 
GSE also consistently matches or exceeds both RigL and Top-KAST, while being strictly more efficient. For example, with VGG-16 on CIFAR-10, GSE achieves 1.4\% and 4.6\% higher accuracy than RigL and Top-KAST, respectively.

Gradually pruning connections during training (Gradual) can still improve the accuracy, especially in the extreme sparsity regimes. However, this comes at the cost of training the dense model, thus requiring significantly more FLOPs than our method. 
This observation serves as motivation for further research into dynamic sparse training methods that improve our accuracy while preserving our efficiency advantage.
Notably, our method outperforms Lottery on the CNN models, while Lottery requires multiple rounds of training.
In Appendix~\ref{sec:training-progress}, we report the test accuracy of each method at every epoch throughout training, showing that our method trains as fast as the other sparse training methods.

\subsection{ImageNet}
\label{sec:imagenet}

To determine the efficacy and robustness of GSE on large models and datasets, in this third experiment, we compare the accuracy of ResNet-50 trained on ImageNet at 90\% sparsity. 
In Table~\ref{tab:imagenet-accuracy}, the accuracy obtained with GSE is compared against the baselines from the preceding section, with the additions of DSR~\citep{mostafa2019parameter} and SNFS~\citep{dettmers2019sparse}. 
Since this is a common benchmark among sparse training methods, the accuracy results of the baselines were obtained from prior publications. 

\begin{table}[h]
    \footnotesize
    \centering
    \caption{ResNet-50 accuracy on ImageNet at 90\% sparsity.}
    \label{tab:imagenet-accuracy}
    \begin{sc}
    \begin{tabular}{l|l}
    \toprule
        Method & Accuracy \\
    \midrule
        Dense & 76.8±0.09 \\
        Lottery & 69.4±0.03 \\
        Gradual & 73.9 \\
        \midrule
        \midrule
        Static & 67.7±0.12 \\
        SNIP & 67.2±0.12 \\
        GraSP & 68.1 \\
        SET & 69.6±0.23 \\
        DSR & 71.6 \\
        SNFS & 72.9±0.06 \\
        RigL & 73.0±0.04 \\
        Top-KAST & 73.0 \\
        GSE (\textit{ours}) & \textbf{73.2±0.07} \\
    \bottomrule
    \end{tabular}
    \end{sc}
\end{table}

Our experiments use the aforementioned hyperparameters with the following exceptions: we train for 100 epochs and use learning rate warm-up over the initial 5 epochs, gradually reaching a value of 0.1; 
subsequently, the learning rate is reduced by a factor of 10 every 30 epochs; 
and we use label smoothing with a coefficient of 0.1. 
Our results use GSE with uniform sampling and $\gamma = 2$. 
The results in Table~\ref{tab:imagenet-accuracy} show a 0.2\% improvement in model accuracy, which is the largest increase since the introduction of SNFS \citep{dettmers2019sparse} in 2019 and sets a new state-of-the-art in dynamic sparse training. These results align with those for CIFAR, and further underscore GSE's performance over other sparse training methods. Note that GSE not only obtains a higher accuracy, it is also more efficient than methods such as SNFS, RigL, and Top-KAST.
These results collectively reinforce GSE's robustness and its ability to scale to large models and datasets.

\subsection{Model scaling}
\label{sec:value-dimensions}

In this experiment, we investigate how the size of a model affects the accuracy when the number of active connections is kept constant. We scaled the number of filters in convolution layers and units in feedforward layers to obtain a wider model. 
The results, presented in Table~\ref{tab:model-scaling}, show that training wider CNN models translates to significantly higher accuracy, 6.5\% on average at 4 times wider. These findings support those of Zhu et al. \citep{zhu2017prune} and Kalchbrenner et al. \citep{kalchbrenner2018efficient}, however, our method enables a sparse model to be trained whose dense version would be too computationally intensive.

\begin{table}[h]
    \centering
    \caption{Accuracy on CIFAR-100 while simultaneously increasing the model width and sparsity. The models in each row have a fixed number of active connections.}
    \label{tab:model-scaling}
    \begin{sc}
    \begin{tabular}{l|cccc}
    \toprule
        Width / Sparsity & 1x / 0\% & 2x / 75\%  & 3x / 89\% & 4x / 94\% \\
    \midrule
        ResNet-56 & 70.1±0.4 & 72.3±0.4 & 73.8±0.6 & 74.2±0.3 \\
        VGG-16 & 62.3±0.8 & 68.4±0.7 & 70.4±0.8 & 71.2±0.3 \\
        ViT & 58.8±0.5 & 58.3±0.2 & 57.3±0.7 & 56.6±0.1 \\
    \bottomrule
    \end{tabular}
    \end{sc}
\end{table}

The same trend does not hold for ViT, as its accuracy decreases with increasing scale.
Appendix~\ref{sec:scale-training-progress} reveals that the accuracy increases slower for larger ViT models, while for CNNs the accuracy increases faster as they get larger.
Moreover, the accuracy of larger dense ViT models plateaus at 62.2\%, as shown in Appendix~\ref{sec:scaling-dense-vit}, indicating that this phenomenon is independent of our method.
ViT models are known to require large amounts of training data to achieve optimal accuracy~\citep{dosovitskiy2020image}, we conjecture that this contributes to the trend observed with ViT.

\subsection{Extended training}
\label{sec:extended-training}

We continued our experimentation by looking at the effect of extended training.
The following extended training results are for CIFAR-100 at 98\% sparsity with 1x, 1.5x, and 2x the number of training epochs. The learning rate drops and the final prune-grow step are also scaled by the same amount, all other hyperparameters stay the same. We observe that ViT and CNN models benefit considerably from extended training, with an average increase in accuracy of 5.3\% for training twice as long. The ViT model in particular gains 7\% in accuracy by training twice as long.

\begin{table}[h]
    \centering
    \caption{Accuracy of extended training on CIFAR-100 at 98\% sparsity.}
    \label{tab:extended-training}
    \begin{sc}
    \begin{tabular}{l|ccc}
        \toprule
        Epochs & 200 & 300 & 400 \\
        \midrule
        ResNet-56 & 54.3±0.5 & 56.4±0.3 & 57.9±0.1 \\
        VGG-16 & 40.8±0.3 & 44.1±0.4 & 45.9±0.2 \\
        ViT & 40.7±1.0 & 45.3±1.1 & 47.8±0.5 \\
        \bottomrule
    \end{tabular}
    \end{sc}
\end{table}

\subsection{Floating-point operations}
\label{sec:flops}

Lastly, to determine the training efficiency of GSE, we compare its FLOPs with those of RigL in Figure~\ref{fig:flops-ratio}. 
FLOPs are determined for ResNet-50 on ImageNet using the method described by Evci et al. \citep{evci2020rigging}, with $\gamma = 1$.
The results show that the benefit of GSE increases rapidly at high sparsity levels. At 99\% sparsity, GSE uses 13.3\% fewer FLOPs than RigL, up from 2.4\% at 95\% sparsity. 

\begin{figure}[h]
    \centering
    \scalebox{1}{\subimport*{figures/}{flops-ratio-3.pgf}}
    \caption{
    Improvement in training FLOPs by GSE compared to RigL for ResNet-50 on ImageNet.
    }
    \label{fig:flops-ratio}
\end{figure}

\section{Conclusion}
\label{sec:conclusion}

We presented an always-sparse training algorithm using guided stochastic exploration of the connections which has a linear time complexity with respect to the model width and improves upon the accuracy of previous sparse training methods. 
This is achieved by only evaluating gradients for a subset of randomly sampled connections when changing the connections. 
The three distributions for sampling the subset---uniform, an upper bound of the gradient magnitude, and an estimate of the gradient magnitude---all showed similar trends. 
However, uniform consistently achieves among the highest accuracy and is also the most efficient. 
We therefore conclude that the uniform distribution is the most appropriate for sparse training. 
Interestingly, we showed that training larger and sparser CNN models results in higher accuracy for the same number of active connections. 
Moreover, it was observed that our method greatly benefits from extended training.
Lastly, the number of floating-point operations of our method decrease considerably faster than those for RigL as the sparsity increases.
This observation underscores the efficiency and scalability of our method, which is promising for training increasingly larger and sparser ANNs.

\bibliographystyle{unsrtnat}
{\linespread{0.9}\selectfont\footnotesize\bibliography{references_compact}}

\begin{thebibliography}{74}
\providecommand{\natexlab}[1]{#1}
\providecommand{\url}[1]{\texttt{#1}}
\expandafter\ifx\csname urlstyle\endcsname\relax
  \providecommand{\doi}[1]{doi: #1}\else
  \providecommand{\doi}{doi: \begingroup \urlstyle{rm}\Url}\fi

\bibitem[Voulodimos et~al.(2018)Voulodimos, Doulamis, et~al.]{voulodimos2018deep}
Athanasios Voulodimos, Nikolaos Doulamis, et~al.
\newblock Deep learning for computer vision: A brief review.
\newblock \emph{Computational intelligence and neuroscience}, 2018, 2018.

\bibitem[O’Mahony et~al.(2019)O’Mahony, Campbell, et~al.]{o2019deep}
Niall O’Mahony, Sean Campbell, et~al.
\newblock Deep learning vs. traditional computer vision.
\newblock In \emph{Science and information conference}. Springer, 2019.

\bibitem[Young et~al.(2018)Young, Hazarika, et~al.]{young2018recent}
Tom Young, Devamanyu Hazarika, et~al.
\newblock Recent trends in deep learning based natural language processing.
\newblock \emph{ieee Computational intelligenCe magazine}, 13\penalty0 (3), 2018.

\bibitem[Otter et~al.(2020)Otter, Medina, et~al.]{otter2020survey}
Daniel~W Otter, Julian~R Medina, et~al.
\newblock A survey of the usages of deep learning for natural language processing.
\newblock \emph{IEEE transactions on neural networks and learning systems}, 32\penalty0 (2), 2020.

\bibitem[Arulkumaran et~al.(2017)Arulkumaran, Deisenroth, et~al.]{arulkumaran2017deep}
Kai Arulkumaran, Marc~Peter Deisenroth, et~al.
\newblock Deep reinforcement learning: A brief survey.
\newblock \emph{IEEE Signal Processing Magazine}, 34\penalty0 (6), 2017.

\bibitem[Schrittwieser et~al.(2020)Schrittwieser, Antonoglou, et~al.]{schrittwieser2020mastering}
Julian Schrittwieser, Ioannis Antonoglou, et~al.
\newblock Mastering atari, go, chess and shogi by planning with a learned model.
\newblock \emph{Nature}, 588\penalty0 (7839), 2020.

\bibitem[Liu et~al.(2017)Liu, Wang, et~al.]{liu2017survey}
Weibo Liu, Zidong Wang, et~al.
\newblock A survey of deep neural network architectures and their applications.
\newblock \emph{Neurocomputing}, 234, 2017.

\bibitem[Wang et~al.(2020)Wang, Zhao, et~al.]{wang2020recent}
Xizhao Wang, Yanxia Zhao, et~al.
\newblock Recent advances in deep learning.
\newblock \emph{International Journal of Machine Learning and Cybernetics}, 2020.

\bibitem[Zhou et~al.(2020)Zhou, Cui, et~al.]{zhou2020graph}
Jie Zhou, Ganqu Cui, et~al.
\newblock Graph neural networks: A review of methods and applications.
\newblock \emph{AI Open}, 1, 2020.

\bibitem[Neyshabur et~al.(2019)Neyshabur, Li, et~al.]{neyshaburrole}
Behnam Neyshabur, Zhiyuan Li, et~al.
\newblock The role of over-parametrization in generalization of neural networks.
\newblock In \emph{ICLR}, 2019.

\bibitem[Kaplan et~al.(2020)Kaplan, McCandlish, et~al.]{kaplan2020scaling}
Jared Kaplan, Sam McCandlish, et~al.
\newblock Scaling laws for neural language models.
\newblock \emph{arXiv preprint arXiv:2001.08361}, 2020.

\bibitem[Brown et~al.(2020)Brown, Mann, et~al.]{brown2020language}
Tom Brown, Benjamin Mann, et~al.
\newblock Language models are few-shot learners.
\newblock \emph{NeurIPS}, 33, 2020.

\bibitem[Rae et~al.(2021)Rae, Borgeaud, et~al.]{rae2021scaling}
Jack~W Rae, Sebastian Borgeaud, et~al.
\newblock Scaling language models: Methods, analysis \& insights from training gopher.
\newblock \emph{arXiv preprint arXiv:2112.11446}, 2021.

\bibitem[Chowdhery et~al.(2022)Chowdhery, Narang, et~al.]{chowdhery2022palm}
Aakanksha Chowdhery, Sharan Narang, et~al.
\newblock Palm: Scaling language modeling with pathways.
\newblock \emph{arXiv preprint arXiv:2204.02311}, 2022.

\bibitem[Hwang(2018)]{hwang2018computational}
Tim Hwang.
\newblock Computational power and the social impact of artificial intelligence.
\newblock \emph{arXiv preprint arXiv:1803.08971}, 2018.

\bibitem[Ahmed and Wahed(2020)]{ahmed2020democratization}
Nur Ahmed and Muntasir Wahed.
\newblock The de-democratization of ai: Deep learning and the compute divide in artificial intelligence research.
\newblock \emph{arXiv preprint arXiv:2010.15581}, 2020.

\bibitem[Reed(1993)]{reed1993pruning}
Russell Reed.
\newblock Pruning algorithms-a survey.
\newblock \emph{IEEE transactions on Neural Networks}, 4\penalty0 (5), 1993.

\bibitem[Gale et~al.(2019)Gale, Elsen, et~al.]{gale2019state}
Trevor Gale, Erich Elsen, et~al.
\newblock The state of sparsity in deep neural networks.
\newblock \emph{arXiv preprint arXiv:1902.09574}, 2019.

\bibitem[Blalock et~al.(2020)Blalock, Gonzalez~Ortiz, et~al.]{blalock2020state}
Davis Blalock, Jose~J. Gonzalez~Ortiz, et~al.
\newblock What is the state of neural network pruning?
\newblock \emph{Proceedings of machine learning and systems}, 2020.

\bibitem[Hoefler et~al.(2021)Hoefler, Alistarh, et~al.]{hoefler2021sparsity}
Torsten Hoefler, Dan Alistarh, et~al.
\newblock Sparsity in deep learning: Pruning and growth for efficient inference and training in neural networks.
\newblock \emph{J. Mach. Learn. Res.}, 22\penalty0 (241), 2021.

\bibitem[Han et~al.(2015)Han, Pool, et~al.]{han2015learning}
Song Han, Jeff Pool, et~al.
\newblock Learning both weights and connections for efficient neural network.
\newblock \emph{NeurIPS}, 28, 2015.

\bibitem[Zhu and Gupta(2017)]{zhu2017prune}
Michael Zhu and Suyog Gupta.
\newblock To prune, or not to prune: exploring the efficacy of pruning for model compression.
\newblock \emph{arXiv preprint arXiv:1710.01878}, 2017.

\bibitem[Kalchbrenner et~al.(2018)Kalchbrenner, Elsen, et~al.]{kalchbrenner2018efficient}
Nal Kalchbrenner, Erich Elsen, et~al.
\newblock Efficient neural audio synthesis.
\newblock In \emph{ICML}. PMLR, 2018.

\bibitem[Janowsky(1989)]{janowsky1989pruning}
Steven~A Janowsky.
\newblock Pruning versus clipping in neural networks.
\newblock \emph{Physical Review A}, 39\penalty0 (12), 1989.

\bibitem[Str{\"o}m(1997)]{strom1997sparse}
Nikko Str{\"o}m.
\newblock Sparse connection and pruning in large dynamic artificial neural networks.
\newblock In \emph{Fifth European Conference on Speech Communication and Technology}. Citeseer, 1997.

\bibitem[Guo et~al.(2016)Guo, Yao, et~al.]{guo2016dynamic}
Yiwen Guo, Anbang Yao, et~al.
\newblock Dynamic network surgery for efficient dnns.
\newblock \emph{NeurIPS}, 29, 2016.

\bibitem[Dong et~al.(2017)Dong, Chen, et~al.]{dong2017learning}
Xin Dong, Shangyu Chen, et~al.
\newblock Learning to prune deep neural networks via layer-wise optimal brain surgeon.
\newblock \emph{NeurIPS}, 30, 2017.

\bibitem[Yu et~al.(2018)Yu, Li, et~al.]{yu2018nisp}
Ruichi Yu, Ang Li, et~al.
\newblock Nips: Pruning networks using neuron importance score propagation.
\newblock In \emph{Proceedings of the IEEE conference on computer vision and pattern recognition}, 2018.

\bibitem[Frankle and Carbin(2018)]{frankle2018lottery}
Jonathan Frankle and Michael Carbin.
\newblock The lottery ticket hypothesis: Finding sparse, trainable neural networks.
\newblock In \emph{ICLR}, 2018.

\bibitem[Frankle et~al.(2019)Frankle, Dziugaite, et~al.]{frankle2019stabilizing}
Jonathan Frankle, Gintare~Karolina Dziugaite, et~al.
\newblock Stabilizing the lottery ticket hypothesis.
\newblock \emph{arXiv preprint arXiv:1903.01611}, 2019.

\bibitem[Lee et~al.(2018)Lee, Ajanthan, et~al.]{lee2018snip}
Namhoon Lee, Thalaiyasingam Ajanthan, et~al.
\newblock Snip: Single-shot network pruning based on connection sensitivity.
\newblock In \emph{ICLR}, 2018.

\bibitem[Tanaka et~al.(2020)Tanaka, Kunin, et~al.]{tanaka2020pruning}
Hidenori Tanaka, Daniel Kunin, et~al.
\newblock Pruning neural networks without any data by iteratively conserving synaptic flow.
\newblock \emph{NeurIPS}, 33, 2020.

\bibitem[Mocanu et~al.(2018)Mocanu, Mocanu, et~al.]{mocanu2018scalable}
Decebal~Constantin Mocanu, Elena Mocanu, et~al.
\newblock Scalable training of artificial neural networks with adaptive sparse connectivity inspired by network science.
\newblock \emph{Nature communications}, 9\penalty0 (1), 2018.

\bibitem[Mostafa and Wang(2019)]{mostafa2019parameter}
Hesham Mostafa and Xin Wang.
\newblock Parameter efficient training of deep convolutional neural networks by dynamic sparse reparameterization.
\newblock In \emph{ICML}. PMLR, 2019.

\bibitem[Evci et~al.(2020)Evci, Gale, et~al.]{evci2020rigging}
Utku Evci, Trevor Gale, et~al.
\newblock Rigging the lottery: Making all tickets winners.
\newblock In \emph{ICML}. PMLR, 2020.

\bibitem[Jayakumar et~al.(2020)Jayakumar, Pascanu, et~al.]{jayakumar2020top}
Siddhant Jayakumar, Razvan Pascanu, et~al.
\newblock Top-kast: Top-k always sparse training.
\newblock \emph{NeurIPS}, 33, 2020.

\bibitem[Jaderberg et~al.(2014)Jaderberg, Vedaldi, et~al.]{jaderberg2014speeding}
Max Jaderberg, Andrea Vedaldi, et~al.
\newblock Speeding up convolutional neural networks with low rank expansions.
\newblock \emph{arXiv preprint arXiv:1405.3866}, 2014.

\bibitem[Novikov et~al.(2015)Novikov, Podoprikhin, et~al.]{novikov2015tensorizing}
Alexander Novikov, Dmitrii Podoprikhin, et~al.
\newblock Tensorizing neural networks.
\newblock \emph{NeurIPS}, 28, 2015.

\bibitem[Gupta et~al.(2015)Gupta, Agrawal, et~al.]{gupta2015deep}
Suyog Gupta, Ankur Agrawal, et~al.
\newblock Deep learning with limited numerical precision.
\newblock In \emph{ICML}. PMLR, 2015.

\bibitem[Mishra et~al.(2018)Mishra, Nurvitadhi, et~al.]{mishra2018wrpn}
Asit Mishra, Eriko Nurvitadhi, et~al.
\newblock Wrpn: Wide reduced-precision networks.
\newblock In \emph{ICLR}, 2018.

\bibitem[Wang et~al.(2019)Wang, Zhang, et~al.]{wang2019picking}
Chaoqi Wang, Guodong Zhang, et~al.
\newblock Picking winning tickets before training by preserving gradient flow.
\newblock In \emph{ICLR}, 2019.

\bibitem[Thimm and Fiesler(1995)]{thimm1995evaluating}
Georg Thimm and Emile Fiesler.
\newblock Evaluating pruning methods.
\newblock In \emph{Proceedings of the International Symposium on Artificial neural networks}, 1995.

\bibitem[Mozer and Smolensky(1988)]{mozer1988skeletonization}
Michael~C Mozer and Paul Smolensky.
\newblock Skeletonization: A technique for trimming the fat from a network via relevance assessment.
\newblock \emph{NeurIPS}, 1, 1988.

\bibitem[Karnin(1990)]{karnin1990simple}
Ehud~D Karnin.
\newblock A simple procedure for pruning back-propagation trained neural networks.
\newblock \emph{IEEE transactions on neural networks}, 1990.

\bibitem[Molchanov et~al.(2019{\natexlab{a}})Molchanov, Tyree, et~al.]{molchanov2019pruning}
P~Molchanov, S~Tyree, et~al.
\newblock Pruning convolutional neural networks for resource efficient inference.
\newblock In \emph{5th ICLR, ICLR 2017-Conference Track Proceedings}, 2019{\natexlab{a}}.

\bibitem[Molchanov et~al.(2019{\natexlab{b}})Molchanov, Mallya, et~al.]{molchanov2019importance}
Pavlo Molchanov, Arun Mallya, et~al.
\newblock Importance estimation for neural network pruning.
\newblock In \emph{Proceedings of the IEEE/CVF Conference on Computer Vision and Pattern Recognition}, 2019{\natexlab{b}}.

\bibitem[LeCun et~al.(1989)LeCun, Denker, et~al.]{lecun1989optimal}
Yann LeCun, John Denker, et~al.
\newblock Optimal brain damage.
\newblock \emph{NeurIPS}, 2, 1989.

\bibitem[Hassibi and Stork(1992)]{hassibi1992second}
Babak Hassibi and David Stork.
\newblock Second order derivatives for network pruning: Optimal brain surgeon.
\newblock \emph{NeurIPS}, 5, 1992.

\bibitem[Frantar et~al.(2021)Frantar, Kurtic, et~al.]{frantar2021m}
Elias Frantar, Eldar Kurtic, et~al.
\newblock M-fac: Efficient matrix-free approximations of second-order information.
\newblock \emph{NeurIPS}, 34, 2021.

\bibitem[Liu et~al.(2021{\natexlab{a}})Liu, Chen, et~al.]{liu2021sparse}
Shiwei Liu, Tianlong Chen, et~al.
\newblock Sparse training via boosting pruning plasticity with neuroregeneration.
\newblock \emph{NeurIPS}, 34, 2021{\natexlab{a}}.

\bibitem[Kingma et~al.(2015)Kingma, Salimans, et~al.]{kingma2015variational}
Durk~P Kingma, Tim Salimans, et~al.
\newblock Variational dropout and the local reparameterization trick.
\newblock \emph{NeurIPS}, 28, 2015.

\bibitem[Molchanov et~al.(2017)Molchanov, Ashukha, et~al.]{molchanov2017variational}
Dmitry Molchanov, Arsenii Ashukha, et~al.
\newblock Variational dropout sparsifies deep neural networks.
\newblock In \emph{ICML}. PMLR, 2017.

\bibitem[Louizos et~al.(2018)Louizos, Welling, et~al.]{louizos2018learning}
Christos Louizos, Max Welling, et~al.
\newblock Learning sparse neural networks through l\_0 regularization.
\newblock In \emph{ICLR}, 2018.

\bibitem[Yang et~al.(2019)Yang, Wen, et~al.]{yang2019deephoyer}
Huanrui Yang, Wei Wen, et~al.
\newblock Deephoyer: Learning sparser neural network with differentiable scale-invariant sparsity measures.
\newblock In \emph{ICLR}, 2019.

\bibitem[Schwarz et~al.(2021)Schwarz, Jayakumar, et~al.]{schwarz2021powerpropagation}
Jonathan Schwarz, Siddhant Jayakumar, et~al.
\newblock Powerpropagation: A sparsity inducing weight reparameterisation.
\newblock \emph{NeurIPS}, 34, 2021.

\bibitem[Liu et~al.(2018)Liu, Sun, et~al.]{liu2018rethinking}
Zhuang Liu, Mingjie Sun, et~al.
\newblock Rethinking the value of network pruning.
\newblock In \emph{ICLR}, 2018.

\bibitem[Lubana and Dick(2020)]{lubana2020gradient}
Ekdeep~Singh Lubana and Robert Dick.
\newblock A gradient flow framework for analyzing network pruning.
\newblock In \emph{ICLR}, 2020.

\bibitem[Neyshabur et~al.(2015)Neyshabur, Tomioka, et~al.]{neyshabur2015norm}
Behnam Neyshabur, Ryota Tomioka, et~al.
\newblock Norm-based capacity control in neural networks.
\newblock In \emph{Conference on Learning Theory}. PMLR, 2015.

\bibitem[Lee et~al.(2019)Lee, Ajanthan, et~al.]{lee2019signal}
Namhoon Lee, Thalaiyasingam Ajanthan, et~al.
\newblock A signal propagation perspective for pruning neural networks at initialization.
\newblock In \emph{ICLR}, 2019.

\bibitem[Bellec et~al.(2018)Bellec, Kappel, et~al.]{bellec2018deep}
Guillaume Bellec, David Kappel, et~al.
\newblock Deep rewiring: Training very sparse deep networks.
\newblock In \emph{ICLR}, 2018.

\bibitem[Dettmers and Zettlemoyer(2019)]{dettmers2019sparse}
Tim Dettmers and Luke Zettlemoyer.
\newblock Sparse networks from scratch: Faster training without losing performance.
\newblock \emph{arXiv preprint arXiv:1907.04840}, 2019.

\bibitem[Liu et~al.(2021{\natexlab{b}})Liu, Yin, et~al.]{liu2021we}
Shiwei Liu, Lu~Yin, et~al.
\newblock Do we actually need dense over-parameterization? in-time over-parameterization in sparse training.
\newblock In \emph{ICML}. PMLR, 2021{\natexlab{b}}.

\bibitem[Vose(1991)]{vose1991linear}
Michael~D Vose.
\newblock A linear algorithm for generating random numbers with a given distribution.
\newblock \emph{IEEE Transactions on Software Engineering}, 17\penalty0 (9), 1991.

\bibitem[Erd{\H{o}}s and R{\'e}nyi(1959)]{erdos59a}
Paul Erd{\H{o}}s and Alfr{\'e}d R{\'e}nyi.
\newblock On random graphs i.
\newblock \emph{Publicationes Mathematicae Debrecen}, 6, 1959.

\bibitem[Evci et~al.(2019)Evci, Pedregosa, et~al.]{evci2019difficulty}
Utku Evci, Fabian Pedregosa, et~al.
\newblock The difficulty of training sparse neural networks.
\newblock \emph{arXiv preprint arXiv:1906.10732}, 2019.

\bibitem[Musser(1997)]{musser1997introspective}
David~R Musser.
\newblock Introspective sorting and selection algorithms.
\newblock \emph{Software: Practice and Experience}, 27\penalty0 (8), 1997.

\bibitem[Krizhevsky et~al.(2009)Krizhevsky, Hinton, et~al.]{krizhevsky2009learning}
Alex Krizhevsky, Geoffrey Hinton, et~al.
\newblock Learning multiple layers of features from tiny images.
\newblock 2009.

\bibitem[Deng et~al.(2009)Deng, Dong, et~al.]{deng2009imagenet}
Jia Deng, Wei Dong, et~al.
\newblock Imagenet: A large-scale hierarchical image database.
\newblock In \emph{2009 IEEE conference on computer vision and pattern recognition}. Ieee, 2009.

\bibitem[He et~al.(2016)He, Zhang, et~al.]{he2016deep}
Kaiming He, Xiangyu Zhang, et~al.
\newblock Deep residual learning for image recognition.
\newblock In \emph{Proceedings of the IEEE conference on computer vision and pattern recognition}, 2016.

\bibitem[Simonyan and Zisserman(2014)]{simonyan2014very}
Karen Simonyan and Andrew Zisserman.
\newblock Very deep convolutional networks for large-scale image recognition.
\newblock \emph{arXiv preprint arXiv:1409.1556}, 2014.

\bibitem[Dosovitskiy et~al.(2020)Dosovitskiy, Beyer, et~al.]{dosovitskiy2020image}
Alexey Dosovitskiy, Lucas Beyer, et~al.
\newblock An image is worth 16x16 words: Transformers for image recognition at scale.
\newblock \emph{arXiv preprint arXiv:2010.11929}, 2020.

\bibitem[Beyer et~al.(2022)Beyer, Zhai, et~al.]{beyer2022better}
Lucas Beyer, Xiaohua Zhai, et~al.
\newblock Better plain vit baselines for imagenet-1k.
\newblock \emph{arXiv preprint arXiv:2205.01580}, 2022.

\bibitem[Alon et~al.(1996)Alon, Matias, et~al.]{alon1996space}
Noga Alon, Yossi Matias, et~al.
\newblock The space complexity of approximating the frequency moments.
\newblock In \emph{Proceedings of the twenty-eighth annual ACM symposium on Theory of computing}, 1996.

\bibitem[Alon et~al.(1999)Alon, Gibbons, et~al.]{alon1999tracking}
Noga Alon, Phillip~B Gibbons, et~al.
\newblock Tracking join and self-join sizes in limited storage.
\newblock In \emph{Proceedings of the eighteenth ACM SIGMOD-SIGACT-SIGART symposium on Principles of database systems}, 1999.

\end{thebibliography}

\newpage
\onecolumn
\appendix

\subsection{Connection sampling distributions}
\label{sec:sampling-distributions}

In this section we provide a more detailed discussion on the distributions which were considered for sampling the subset of inactive connections.

\subsubsection{Gradient magnitude upper bound (GraBo)}
\label{sec:grabo}

The gradient magnitude for a single connection $(a, b) \in W$ can be expressed as follows:
\begin{align}
     \abs{\nabla \theta_{l}(a,b)} &= \abs{\sum_{i=1}^B \vh_{l-1}(a, i)\vdelta_{l}(b, i)}\\ \implies \abs{\nabla \theta_{l}} &= \abs{\vh_{l-1}\trans{\vdelta_{l}}}
\end{align}
where $\vh_{l}$ and $\vdelta_{l}$ are the activations and gradients of the loss at the output units of the $l$-th layer, respectively. When the batch size $B$ is one, the sum disappears, simplifying the gradient magnitude to a rank one matrix which can be used to sample connections efficiently,
inducing a joint probability distribution that is proportional to the gradient magnitude only when the batch size is one. In practice, however, training samples come in mini batches to reduce the variance of the gradient. 
Therefore, we experiment with sampling the connections proportionally to the following upper bound of the gradient magnitude instead, which enables efficient sampling, even with mini batches:
\begin{align}
    \abs{\nabla \theta_{l}} = \abs{\vh_{l-1}\trans{\vdelta_{l}}} \leq \left(\abs{\vh_{l-1}} \vone\right) \trans{\left(\abs{\vdelta_{l}} \vone\right)} \label{eq:grabo}
\end{align}
The proof for Equation~\ref{eq:grabo} uses the triangle inequality and is provided in Appendix~\ref{sec:proofs}. This upper bound has the property that it becomes an equality when $B=1$. Connections are then sampled from the probability distributions in Equation~\ref{eq:grabo-dist}. The implementation of this distribution does not require any modifications to the model, and only needs a single forward and backward pass to evaluate.
\begin{align}
    \rf_{l} \coloneqq \frac{\abs{\vh_{l-1}}\vone}{\trans{\vone}\abs{\vh_{l-1}}\vone}, \quad \rg_{l} \coloneqq \frac{\abs{\vdelta_{l}}\vone}{ \trans{\vone}\abs{\vdelta_{l}}\vone} \label{eq:grabo-dist}
\end{align}

\subsubsection{Gradient magnitude estimate (GraEst)}
\label{sec:ams}

The reason we use a bound on the absolute gradient is that in general the sum of products cannot be written as the product of sums, that is, $\sum_i x_iy_i \neq \left(\sum_i x_i\right) \left(\sum_i y_i\right)$. However, it is possible to satisfy this equality in expectation. This was shown by Alon et al. \citep{alon1996space} who used it to estimate the second frequency moment, or the $\normltwo$ norm of a vector. Their work is called the AMS sketch, AMS is an initialism derived from the last names of the authors. They then showed that this also works in the more general case to estimate inner products between vectors \citep{alon1999tracking}. This is achieved by assigning each index of the vector a random sign, i.e. $\evs_i \sim \mathcal{U}\{-1, +1\}$, resulting in the following equality:
\begin{align}
    \sum_i x_iy_i = \E\left[\left(\sum_i \evs_ix_i\right)\left(\sum_i \evs_iy_i\right)\right] \label{eq:ams}
\end{align}
We can use their finding to construct probability distributions $\rf_{l}$ and $\rg_{l}$ in Equation~\ref{eq:ams-dist} such that their induced joint probability distribution is proportional to the gradient magnitude in expectation. This distribution can also be computed in a single forward and backward pass and requires no modifications to the model, apart from introducing a mechanism to sample the random signs $\vs$.
\begin{align}
        \rf_{l} \coloneqq \frac{\abs{\vh_{l-1}\vs}}{\trans{\vone}\abs{\vh_{l-1}\vs}}, \quad \rg_{l} \coloneqq \frac{\abs{\vdelta_{l}\vs}}{ \trans{\vone}\abs{\vdelta_{l}\vs}} \label{eq:ams-dist}
\end{align}

\subsection{Gradient magnitude upper bound proof}
\label{sec:proofs}

\begin{lemma}
The gradient magnitude of the loss with respect to the parameters of a feedforward layer $l$ has the following upper bound:
\begin{align*}
    \abs{\nabla \theta_{l}} \leq \left(\abs{\vh_{l-1}} \vone\right) \trans{\left(\abs{\vdelta_{l}} \vone\right)}
\end{align*}
where $\theta_{l} \in \mathbb{R}^{n_{l-1}\times n_{l}}$, $\vh_{l} \in \mathbb{R}^{n_{l}\times B}$, and $\vdelta_{l}  \in \mathbb{R}^{n_{l} \times B}$ are the weight matrix, activations and gradients of the loss at the output units of the $l$-th layer, respectively. The number of units in the $l$-th layer is denoted by $n_{l} \in \mathbb{N}^{+}$, and $B \in \mathbb{N}^{+}$ is the batch size.
\end{lemma}

\begin{proof}
The proof starts with the definition of the gradient of the feedforward layer, which is calculated during back propagation as the activations of the previous layer multiplied with the back propagated gradient of the output units of the layer. We then use the triangle inequality to specify the first upper bound in Equation~\ref{eq:proof-2}.
\begin{align}
    \abs{\nabla \theta_{l}(a,b)} &= \abs{\sum_{i=1}^B \vh_{l-1}(a, i)\vdelta_{l}(b, i)}\\ &\leq \sum_{i=1}^B \abs{\vh_{l-1}(a, i)\vdelta_{l}(b, i)} \label{eq:proof-2}\\
    &= \sum_i^B \abs{\vh_{l-1}(a, i)} \sum_i^B \abs{\vdelta_{l}(b,i)} \\ & \qquad - \sum_{r, s \neq r}^B \abs{\vh_{l-1}(a,r)\vdelta_{l}(b,s)} \\&\leq \sum_i^B \abs{\vh_{l-1}(a,i)} \sum_i^B \abs{\vdelta_{l}(b,i)} \label{eq:proof-3}
\end{align}
This implies that $\abs{\nabla \theta_{l}} \leq \left(\abs{\vh_{l-1}} \vone\right) \trans{\left(\abs{\vdelta_{l}} \vone\right)}$.
\end{proof}

\subsection{Data normalization}
\label{sec:data-normalization}

We normalize all the training and test data to have zero mean and a standard deviation of one. The dataset statistics that we used are specified below, where $\mu$ is the mean and $\sigma$ the standard deviation. The values correspond to the red, green, and blue color channels of the image, respectively.

\paragraph{CIFAR-10}
\begin{align*}
    \mu= (0.491, 0.482, 0.447), \quad \sigma = (0.247, 0.243, 0.262)
\end{align*}

\paragraph{CIFAR-100}
\begin{align*}
    \mu= (0.507, 0.487, 0.441), \quad \sigma = (0.267, 0.256, 0.276)
\end{align*}

\paragraph{ImageNet}
\begin{align*}
    \mu= (0.485, 0.456, 0.406), \quad \sigma = (0.229, 0.224, 0.225)
\end{align*}

\subsection{Data augmentation}
\label{sec:data-augmentation}

We applied standard data augmentation as part of our training data pipeline. The specific augmentation techniques are specified below. We apply the data normalization described in Appendix~\ref{sec:data-normalization} after the augmentation transformations.

\paragraph{CIFAR-10}
We pad the image with 4 black pixels on all sides and randomly crop a square of 32 by 32 pixels. We then perform a random horizontal flip of the crop with a probability of 0.5.

\paragraph{CIFAR-100}
We pad the image with 4 black pixels on all sides and randomly crop a square of 32 by 32 pixels. We then perform a random horizontal flip of the crop with a probability of 0.5.

\paragraph{ImageNet}
We randomly crop a square of 224 by 224 pixels. We then perform a random horizontal flip of the crop with a probability of 0.5. For the test set, we resize the images to be 256 pixels followed by a square center crop of 224 pixels.

\clearpage
\subsection{Convolutional layer subset sampling}
\label{sec:conv-layer-subset}

In this section we describe how our growing and pruning algorithm can be applied to convolutional layers.
While a weight in a feedforward layer is specified by the input and output units that it connects, a weight in a 2D convolutional layer is specified by the input channel, its $x$ and $y$ coordinate on the filter, and the output channel. It thus seems that it would require four discrete probability distributions to sample a weight of a 2D convolutional layer, instead of two for feedforward layers. However, one can alternatively view a convolutional layer as applying a feedforward layer on multiple subsections (patches) of the input, flattening the input channels and filters makes the weights of a convolutional layer identical to a feedforward layer. The multiple subsections of the input can then simply be treated as additional batches. With this realization all the methods that we explain throughout the paper for feedforward layers directly apply to convolutional layers.

\subsection{Sparse graph initialization comparison}
\label{sec:layer-sparsity-distribution}

In Figure~\ref{fig:sparsity-assignment}, we plot the accuracy obtained by assigning the sparsity uniformly over all the layers compared to using the Erdős–Rényi initialization for SET, RigL, and our method. Note that Ours Uniform (Erdős–Rényi) denotes uniform sampling of the subset connections and the Erdős–Rényi initialization. The Erdős–Rényi initialization obtains better accuracy while ensuring that the number of active connections per layer is proportional to the layer width.

\begin{figure*}[h]
    \centering
    \scalebox{0.9}{\subimport*{figures/}{sparsity-assignment.pgf}}
    \caption{Accuracy comparison of uniform and Erdős–Rényi sparsity assignment.}
    \label{fig:sparsity-assignment}
\end{figure*}

\clearpage
\subsection{Pseudocode}
\label{sec:pseudocode}

We provide pseudocode for our complete dynamic sparse training process in the case of the GraBo or GraEst distributions in Algorithm~\ref{alg:full-dst}. The uniform distribution is simpler as it does not require any aggregation of activity. For simplification, the pseudocode assumes that the prune-grow step only uses a single batch of training samples, allowing the aggregation for the probability distributions $\rf$ and $\rg$ to be computed in the same pass as the generation of the various connection sets. This is also the scenario we tested with in all our experiments.

\begin{algorithm}[h]
\caption{Efficient periodic growing and pruning of connections}\label{alg:full-dst}
\begin{algorithmic}
\STATE {\textbf{Input:}} Network $f_\theta$, dataset $\mathcal{D}$, loss function $\mathcal{L}$, prune-grow schedule $T, T_{\mathrm{end}}, \alpha$, number of active connections multiplier $\epsilon$, and subset sampling factor $\gamma$.
\FOR{each layer $l$}
\STATE $\sA \gets \mathrm{sample\_graph}(n^{[l-1]}, n^{[l]}; \epsilon)$
\STATE $\theta \gets \mathrm{sample\_weights}(\sA)$
\ENDFOR
\FOR{each training step $t$} 
\STATE $\vx, \vy \sim \mathcal{D}$
\IF{$t \bmod{T} = 0$ \textbf{and} $t \leq T_{\mathrm{end}}$}
\STATE $\vh \gets \vx$
\FOR{each layer $l$}
\STATE $\va^{[l]} \gets \mathrm{aggregate}(\vh)$
\STATE $\vh \gets \mathrm{forward}(\vh)$
\ENDFOR
\STATE $\vg \gets \mathrm{grad}(\loss{\vh, \vy})$
\FOR{each layer $l$ in inverse order}
\STATE $\vb^{[l]} \gets \mathrm{aggregate}(\vg)$
\STATE $\rf \gets \mathrm{normalize}(\va^{[l]})$
\STATE $\rg \gets \mathrm{normalize}(\vb^{[l]})$ 
\STATE $\sS \gets \mathrm{sample\_connections}(\rf, \rg, \ceil{\gamma\abs{\sA}}) \setminus \sA$
\STATE $\alpha_t \gets \mathrm{cosine\_decay}(t; \alpha, T_{\mathrm{end}})$
\STATE $k \gets \min(\ceil{\alpha_t\abs{\sA}}, \abs{\sS})$
\STATE $\sG \gets \mathrm{topk}(\abs{\mathrm{grad}(\sS)}, k)$
\STATE $\sP \gets \mathrm{topk}(-\abs{\theta}, k)$
\STATE $\sA \gets (\sA \setminus \sP) \cup \sG$
\STATE $\theta \gets \mathrm{update\_weights}(\theta, \sA, \sP, \sG)$
\STATE $\vg \gets \mathrm{backward}(\vg)$
\ENDFOR
\ELSE
\STATE $\mathrm{SGD}(\theta, \loss{f_\theta(\vx), \vy})$
\ENDIF
\ENDFOR
\end{algorithmic}
\end{algorithm}

\subsection{Training progress}
\label{sec:training-progress}

In Figure~\ref{fig:train-progress}, we provide plots for the test accuracy during training for each baseline method and our method. It can be seen that our method consistently achieves among the highest test accuracy throughout training compared to the sparse training methods. The training progress of Gradual stands out because its accuracy increases the most at first, this is because it starts out as a dense model while it gradually increases the model sparsity, as a result, the accuracy goes down towards the second learning rate drop when the target sparsity is reached.

\begin{figure*}[h]
    \centering
    \scalebox{0.8}{\subimport*{figures/}{method-progress-s98.pgf}}
    \caption{Training progress in test accuracy compared to related work at 98\% sparsity.}
    \label{fig:train-progress}
\end{figure*}

\subsection{Model scaling training progress}
\label{sec:scale-training-progress}

In Figure~\ref{fig:scale-train-progress}, we provide plots for the test and training accuracy during training for our method on CIFAR-100 while increasing the width of the model by 1x, 2x, 3x, and 4x. The number of active connections is kept constant between all experiments, see Section~\ref{sec:value-dimensions} for details. In the CNN models, ResNet-56 and VGG-16, the increase in width causes the models to train significantly faster and obtain higher accuracy. For the vision transformer (ViT) model, the trend is reversed, wider models train slower and obtain lower accuracy.

\begin{figure*}[h]
    \centering
    \scalebox{1}{\subimport*{figures/}{width-training-cifar100.pgf}}
    \caption{Training progress in train and test accuracy for various model widths on CIFAR-100.}
    \label{fig:scale-train-progress}
\end{figure*}

\subsection{Scaling dense ViT}
\label{sec:scaling-dense-vit}

To further investigate the behavior of ViT, which decreases in accuracy when using larger models with the same number of activate parameters, as discussed in Section~\ref{sec:value-dimensions} and Appendix~\ref{sec:scale-training-progress}. In Table~\ref{tab:scaling-dense-vit}, we report the accuracy of ViT for increasing model width but kept dense, thus increasing the active connections count. We see the accuracy quickly plateaus at 62.2\%, gaining only 2\% accuracy while even the sparse CNNs gained 6.5\% on average at 4 times wider. This indicates that the decreasing accuracy of ViT reported in Section~\ref{sec:value-dimensions} is not due to our sparse training method as even the dense model does not benefit much from increased scale.

\begin{table}[h]
    \centering
    \caption{Accuracy of scaling dense ViT on CIFAR-100}
    \label{tab:scaling-dense-vit}
    \begin{sc}
    \begin{tabular}{l|cccc}
    \toprule
	Model width & 1x & 2x & 3x & 4x \\
    \midrule
    Accuracy & 60.3±0.6 & 62.5±0.5 & 62.2±0.4 & 62.2±0.1 \\		
    \bottomrule
    \end{tabular}
    \end{sc}
\end{table}

\subsection{Sparse matrix multiply benchmark}
\label{sec:sparse-benchmark}

In Figure~\ref{fig:sparse-ops-timing}, we show the execution time of the sparse times dense matrix-matrix multiply on two execution devices: Intel Xeon Gold 6148 CPU at 2.40GHz and NVIDIA Tesla V100 16GB. The sparse matrix is a square matrix of size units times units, representing an unstructured sparse random weight matrix. The dense matrix is of size units times 128 and represents a batch of layer inputs. We compare three storage formats, namely: dense, where all the zeros are explicitly represented; Coordinate format (COO), which represents only the nonzero elements and their coordinates; and Compressed Sparse Row (CSR), which is similar to COO but compresses the row indices. Each data point in Figure~\ref{fig:sparse-ops-timing} is the average of 10 measurements. 

\begin{figure*}[h]
    \centering
    \scalebox{0.9}{\subimport*{figures/}{sparse-ops-timing.pgf}}
    \caption{Execution time of sparse times dense matrix-matrix multiply.}
    \label{fig:sparse-ops-timing}
\end{figure*}

The results show that the CSR format generally outperforms the COO format. Moreover, the execution time of the CSR format improves upon the dense execution time starting around 90\% sparsity. At 99\% sparsity, the CSR format is up to an order of magnitude faster than the dense format. These results are similar between the two execution devices, although the NVIDIA Tesla V100 16GB is up to two orders of magnitude faster than the Intel Xeon Gold 6148 CPU at 2.40GHz.

\end{document}